\documentclass{article}

\usepackage[utf8]{inputenc} 
\usepackage[T1]{fontenc}    
\usepackage{hyperref}       
\usepackage{url}            
\usepackage{booktabs}       
\usepackage{amsfonts, amsmath}       
\usepackage{nicefrac}       
\usepackage{microtype}      
\usepackage{graphicx}
\usepackage[caption=false]{subfig}
\usepackage{bm}

\usepackage[noend]{algpseudocode}

\usepackage[preprint]{icml2020}

\icmltitlerunning{Lower Bounds on Cross-entropy Loss in the Presence of Test-time Adversaries}

\usepackage{amsmath}
\usepackage{amssymb}
\usepackage{amsfonts}
\usepackage{amsthm}
\usepackage{array}
\usepackage{float}
\usepackage{outlines}
\usepackage{dsfont}
\usepackage{tikz}

\newtheorem{theorem}{Theorem}

\newtheorem{lemma}{Lemma}

\newtheorem{definition}{Definition}

\newcommand{\bcomment}[1]{}
\newcommand{\TODO}[1]{}

\newcommand{\E}{\mathbb{E}}

\newcommand{\R}{\mathbb{R}}

\newcommand{\twobyone}[2]{\begin{pmatrix} #1 \\ #2 \end{pmatrix}}

\newcommand{\diag}[1]{\operatorname{diag}(#1)}

\newcommand{\mcA}{\mathcal{A}}
\newcommand{\mcB}{\mathcal{B}}

\newcommand{\mcE}{\mathcal{E}}

\newcommand{\mcG}{\mathcal{G}}

\newcommand{\mcV}{\mathcal{V}}
\newcommand{\mcX}{\mathcal{X}}
\newcommand{\mcY}{\mathcal{Y}}

\newcommand{\one}{\mathbf{1}}
\newcommand{\zero}{\mathbf{0}}

\newcommand{\multisetds}[2]{\bigg(\kern-.4em\binom{#1}{#2}\kern-.4em\bigg)}
\newcommand{\multisetin}[2]{\big(\kern-.3em\binom{#1}{#2}\kern-.3em\big)}
\newcommand{\multisetix}[2]{\left(\kern-.2em\binom{#1}{#2}\kern-.2em\right)}

\begin{document}

\twocolumn[
\icmltitle{Lower Bounds on Cross-Entropy Loss in the Presence of Test-time Adversaries}



\icmlsetsymbol{equal}{*}

\begin{icmlauthorlist}
\icmlauthor{Arjun Nitin Bhagoji}{equal,uchi}
\icmlauthor{Daniel Cullina}{equal,ps}
\icmlauthor{Vikash Sehwag}{pr}
\icmlauthor{Prateek Mittal}{pr}
\end{icmlauthorlist}

\icmlaffiliation{uchi}{Department of Computer Science, University of Chicago}
\icmlaffiliation{ps}{Department of Electrical and Computer Engineering, Pennsylvania State University}
\icmlaffiliation{pr}{Department of Electrical Engineering, Princeton University}

\icmlcorrespondingauthor{Arjun Nitin Bhagoji}{abhagoji@uchicago.edu}

\icmlkeywords{Machine Learning, ICML}

\vskip 0.3in
]



\printAffiliationsAndNotice{\icmlEqualContribution} 




	

\begin{abstract}
Understanding the fundamental limits of robust supervised learning has emerged as a problem of immense interest, from both practical and theoretical standpoints. In particular, it is critical to determine classifier-agnostic bounds on the training loss to establish when learning is possible. In this paper, we determine optimal lower bounds on the cross-entropy loss in the presence of test-time adversaries, along with the corresponding optimal classification outputs. Our formulation of the bound as a solution to an optimization problem is general enough to encompass any loss function depending on soft classifier outputs. We also propose and provide a proof of correctness for a bespoke algorithm to compute this lower bound efficiently, allowing us to determine lower bounds for multiple practical datasets of interest. We use our lower bounds as a diagnostic tool to determine the effectiveness of current robust training methods and find a gap from optimality at larger budgets. Finally, we investigate the possibility of using of optimal classification outputs as soft labels to empirically improve robust training. 
\end{abstract}

\section{Introduction}\label{sec: intro}
 The robustness of machine learning systems, particularly classifiers in the supervised setting, to adversarial perturbations \citep{szegedy2013intriguing,goodfellow2014explaining,carlini2017towards,bhagoji2018practical,madry_towards_2017} has become an important line of research owing to the critical role they play in society. While there is a tremendous amount of work on attacks and defenses \cite{papernot2016towards}, a focus of recent research \cite{bhagoji2019lower,pmlr-v97-dohmatob19a,schmidt2018adversarially,cullina2018pac,mahloujifar2019curse,diochnos2018adversarial} has been on establishing fundamental bounds on learning in the presence of test-time adversaries in various settings. One line of research \cite{bhagoji2019lower,pmlr-v97-dohmatob19a,pmlr-v119-pydi20a} into the limits of learning in the presence of test-time attackers has established classifier-agnostic lower bounds on adversarial robustness, i.e. the minimum $0-1$ loss that would be incurred by any classifier, when adversarial perturbations are added to the underlying data distribution. However, practical approaches to training classifiers such as neural networks usually use surrogate loss functions such as the cross-entropy loss that depend on the output confidence, and it is critical to establish bounds on these.
 
 
Thus, in this paper, we extend work on the information-theoretic limits of learning in the presence of test-time adversaries to \textit{any} loss function that uses the output probabilities of a classifier, such as the cross-entropy loss. The key question this paper answers is:

\emph{What is the minimum possible cross-entropy loss that will be incurred by any classifier given a data distribution and adversary specification?} 

Answering this question enables us to quantitatively diagnose the effectiveness of practical defenses against adversarial examples~\cite{madry_towards_2017, zhang2019theoretically}, and can inform the design of better learning algorithms. In particular, we can determine if current robust optimization techniques are able to recover these bounds as well as find regimes in which robust classification is not possible.


To determine classifier-agnostic lower bounds on the cross-entropy loss, we focus on the interaction between data points when they are perturbed. We represent points from each class as the vertices of a graph, with edges existing between two vertices if the neighborhoods in which they can be perturbed overlap. We refer to this structure as a \emph{conflict graph}. Data points connected by edges are then challenging to classify, even for the optimal classifier. The problem is then translated to one of finding the output probabilities of the optimal classifier over this graph. Minimizing the cross entropy loss over this graph determines these probabilities and provides a lower bound, which can be efficiently computed as the resulting optimization problem
is convex. This quantity, known as the \emph{graph entropy} \cite{korner1973coding}, has independently appeared in information theory as the solution to a coding problem. We also determine an exact form for the lower bound on cross-entropy for a mixture of two Gaussians, along with the optimal classifer and adversarial strategy.


An efficient determination of these lower bounds is possible since the optimization problem
is convex, but we find that existing solvers are prohibitively slow for the programs resulting from real-world distributions of interest. In light of this, we derive a custom algorithm that exploits the bipartite structure of the conflict graph and can determine bounds far faster than a generic convex solver for instantiations of interest. Our algorithm can find a solution in $10$s of seconds for benchmark datasets such as MNIST \cite{lecun1998mnist}, Fashion MNIST \cite{xiao2017/online} and CIFAR-10 \cite{krizhevsky2009learning}. We provide a proof of correctness and convergence for our algorithm.


We use our algorithm to find lower bounds on the cross-entropy loss for these benchmark datasets, as well as for synthetic Gaussian data. Comparing these bounds to the training loss obtained by state-of-the-art robust optimization techniques on commonly used deep neural networks, we find a gap in terms of convergence to the optimal loss. Interestingly, the gap is much larger for the $0-1$ loss than for the cross-entropy loss, indicating that the use of a surrogate loss does impact achievability but is not the sole reason for it. We examine the impact of model architectures and activation functions on this gap, finding that the former aids convergence while the latter has a negligible impact. Finally, for certain adversarial budgets, we find that the use of soft labels obtained from our framework during training can aid with both convergence and generalization. The code to reproduce all results in this paper is available at \url{https://github.com/arjunbhagoji/log-loss-lower-bounds}.

\subsection{Summary of Contributions}

\noindent \textbf{General framework for lower bounds for all convex losses using output probabilities in the presence of test-time adversaries:}
Our problem formulation allows us to determine lower bounds on any loss function for a given dataset and adversary. In particular, we can compute lower bounds on the commonly used cross-entropy loss as well as the the optimal classification probabilities for all points.

\noindent \textbf{Efficient determination of optimal log-loss}: We propose a bespoke algorithm to compute these lower bounds and provide a proof of its correctness. For practical settings of interest, our algorithm provides a speedup of multiple orders of magnitude over a generic convex solver from CVXOPT \cite{andersen2013cvxopt}.

\noindent \textbf{Analyzing the effectiveness of current robust training methods}: Our framework enables us to determine regimes where robust classification is possible. In these regimes, we find that current robust training techniques are able to get close to, but not match, the lower bounds on cross-entropy loss. This gap is smaller than that for the $0-1$ loss observed in previous work, showing the impact of using surrogate losses. We also investigate the use of the optimal classification probabilities computed by our framework as soft-labels during training, and find that these aid in both convergence and generalization for certain adversarial budgets.

\section{Lower Bounds on Cross-Entropy Loss}\label{sec: log_loss}
In this section, we derive lower bounds on the cross-entropy loss in the presence of a test-time attacker by demonstrating that it is the solution to a convex optimization problem. We show how this problem can be derived using a graphical interpretation of the classification problem. Our method applies to \emph{all discrete two-class distributions as well as all adversaries perturbing points within a non-empty neighborhood}. We also extend our framework to the special case of a mixture of two Gaussians.

\subsection{Problem formulation}

We consider the following supervised classification problem. Data points $x$ are drawn from a space $\mcX$, with labels $y \in \mcY=\{-1,1\}$. The joint probability distribution over this data is $P$. The classification function (or classifier) $f: \mcX \rightarrow \mcY$ maps data points to the space of labels. We also define a `soft' classifier $h: \mcX \rightarrow [0,1]^{\mcY}$ that maps data points to a metric of their confidence of being in a class. The index of the maximum element of $h$ recovers $f$. This approach is followed in classification algorithms such as logistic regression and neural networks \cite{shalev-shwartz_understanding_2014}.

\noindent \textbf{Test-time adversary:} We consider a test-time adversary that can modify any data point to generate an adversarial example \cite{goodfellow2014explaining,szegedy2013intriguing,carlini2017towards} within a neighborhood, i.e. $\tilde{x}=N(x)$, where $\tilde{x}$ is the adversarial example and $N(\cdot)$ is a non-empty neighborhood function. This general definition includes the $\ell_p$ family of constraints most widely used in previous work.

\noindent \textbf{Loss functions and robust training:} To obtain a classifier robust to test-time adversaries, $f$ must be trained to minimize the robust $0-1$ loss, defined as $\mathbb{E} \left[\tilde{\ell}_{0-1}(f,(x,y)) \right]=\mathbb{E} \left[ sup_{\tilde{x} \in N(x)} \bm{1} \left( f(\tilde{x}) \neq y \right) \right]$. However, since the $0-1$ loss is non-differentiable, surrogate losses that are differentiable and upper bound it are used in practice. One of the most common is the cross-entropy or log loss, defined as $\ell^{\text{CE}}(h,v)=-\log h(x)_y$ for a $2$-class problem, where $v=(x,y)$ and $h(x) \in [0,1]^{\mcY}$ is the probability distribution over $\mcY$ that the soft classifier $h$ assigns to $x$. The robust classification problem using a surrogate loss $\ell$ is then
\begin{multline}
    	\min_h \mathbb{E}_P \left[sup_{\tilde{x} \in N(x)} \ell^{\text{CE}}(h,(\tilde{x},y)) \right] \\= \min_h \mathbb{E}_P \left[\tilde{\ell}^{\text{CE}}(h,(x,y)) \right] 
\end{multline}
The robust cross-entropy loss is of particular interest in the robust training of neural networks \cite{madry_towards_2017}.

\noindent \textbf{Problem Statement:} Our aim is to determine the value of $\min_h \mathbb{E}_P [\tilde{\ell}^{\text{CE}}(h,(x,y)) ] $ over all measurable functions $h$, for a given discrete, two class distribution $P$ and neighborhood function $N(\cdot)$.

\subsection{Lower bound as the solution to a convex program}
We first define a conflict graph in order to cast the problem of finding the lower bound as an optimization problem over the vertices of this graph. Then, we show that the feasible set of output probabilities is determined by the edge incidence matrix of the conflict graph. Finally, we determine the lower bound on the cross-entropy loss by minimizing over this feasible set.

\noindent \textbf{Conflict graph:} We define a conflict graph $\mathcal{G}=(\mcV,\mcE)$ that accounts for intersections between the neighborhoods of points from different classes. Each neighborhood represents the set of points reachable by the adversary from point $x$.
Let $\mcV \subseteq \mcX \times \mcY$ be the support of the distribution $P$. This means that each labeled data point $(x,y)$ with strictly positive probability in $P$ is represented as a vertex $v$. Since we consider a binary classification problem, the conflict graph is bipartite. Each part of the graph is $\mcV_c = \mcV \cap(\mcX \times \{c\})$, where $c \in \{-1,1\}$. The edge $((x,1),(x',-1))$ is present if and only if $N((x,1)) \cap  N((x',-1))$ is nonempty. There are no edges between vertices in the same part of the graph.



\TODO{Properly move assumptions from thm to here}

\begin{definition}
For a soft classifier $h$, the correct-classification probability $q_v$ that it can achieve on an example $v = (x,y)$ in the presence of an adversary is
\[
    q_v=\inf_{\tilde{x} \in N(x)} h(\tilde{x})_y.
\]
\end{definition}

\begin{lemma}[Feasible output probabilities]
Let $q \in \R^{\mcV}$ be the vector of correct-classification probabilities obtained by a classifier. 
The feasible set of such probabilities is
\begin{equation}
\begin{aligned}
& q \geq \zero\\
& M q \leq \one. \label{conflict-constaints}
\end{aligned}
\end{equation}
where $M = \twobyone{E}{I} \in \R^{(\mcE \sqcup \mcV) \times \mcV}$ and $E \in \R^{\mcE \times \mcV}$ is the edge incidence matrix of the conflict graph.
\label{lemma: feasible_q}
\end{lemma}

\begin{proof}
  Suppose that $(u,v) \in \mcE$.
  Then, there is some $\tilde{x} \in N(u) \cap N(v)$.
  We have $q_u \leq h(\tilde{x})_1$, $q_v \leq h(\tilde{x})_{-1}$, and $h(\tilde{x})_1 + h(\tilde{x})_{-1} = 1$.
  Combining these gives the constraint in \eqref{conflict-constaints} indexed by $(u,v)$.
  
  Now, we will show that each vector $q$ in the polytope is achievable by some $h$.
  Let $h(\tilde{x})_1 = \sup_{u : \tilde{x} \in N(u)} q_u$ and $h(\tilde{x})_{-1} = 1 - h(\tilde{x})_1$. Then,
  
  $\inf_{\tilde{x} \in N(u)} h(\tilde{x})_{1} = \inf_{\tilde{x} \in N(u)} \sup_{u' : \tilde{x} \in N(u')} q_{u'} \geq \inf_{\tilde{x} \in N(u)} q_u = q_u$

  The output when the true example is $v$ is
  $\inf_{\tilde{x} \in N(v)} h(\tilde{x})_{-1} = \inf_{\tilde{x} \in N(v)} (1 - \sup_{u : \tilde{x} \in N(u)} q_u) = \inf_{u : \exists \tilde{x} \in N(u) \cap N(v)} (1 - q_u) \geq q_v.$
  
\end{proof}

In the non-adversarial case, all constraints in \eqref{conflict-constaints} are of the forms $q_v \leq 1$ and $q_{(x,1)} + q_{(x,-1)} \leq 1$. 
Non-trivial adversaries lead to constraints between the probabilities achieved for distinct examples.

Having determined the feasible set of output probabilities, we can now determine the minimum possible cross-entropy loss by minimizing it over this feasible set.

\begin{theorem}[Lower bound on cross-entropy loss]
	The discrete joint probability distribution $P$ over data from two classes, and the neighborhood function $N(\cdot)$ define a bipartite conflict graph $\mathcal{G}$ with incidence matrix $E$.
	Let $p \in \R^{\mcV}$ with $p_v = P(\{v\})$.
	Let $q^*$ be the minimizer of the following program:
	\begin{equation}
	\begin{aligned}
		\min_{q}& \sum_{v : p_v > 0} - p_v \log q_v \\
		\text{s.t.}& \,  q \geq \zero \\
		& M q \leq \one.
		\label{ce_lower_bound}
	\end{aligned}
	\end{equation}
	Then, there is a classifier $h^*$ that achieves the correct-classification probabilities $q^*$ and for all $h$,
	$\mathbb{E}_{P}[\tilde{\ell}^{\text{CE}}(h^*,v)] \leq \mathbb{E}_{P}[\tilde{\ell}^{\text{CE}}(h,v)]$.
\label{thm: log_loss_lower_bound}
\end{theorem}

\begin{proof}
From Lemma~\ref{lemma: feasible_q}, we know that the constraints in Eq.\eqref{conflict-constaints} represent the feasible set of all possible $q$. Further, there exists some $h$ that achieves each $q$. The objective function must have a minimum in the feasible set. Additionally, the objective is convex and the constraints are linear, leading to a convex program.
\end{proof}

We note that a modification of the program above can be used to derive the minimum $0-1$ loss for discrete distributions by setting $\min_q \sum_{v} p^\intercal q$ as the objective function.

\begin{lemma}[Properties of an optimal $q$]
\label{lemma: convex-duality}
  Suppose we have $q$ and $z$ such that
  \begin{align}
    q &\geq \zero \label{q-pos}\\
    Mq &\leq \one \label{q-pack}\\
    z & \geq \zero \label{z-pos}\\
    \diag{q} M^{\top}z &\geq p \label{z-cover}\\
    \one^{\top}z &\leq \one^{\top}p \label{strong-duality}.
  \end{align}
 Then, $q$ is optimal in \eqref{ce_lower_bound}.
\end{lemma}
\begin{proof}

From \eqref{z-cover} we have $\one^{\top} p \leq q^{\top}M^{\top}z$ and from the \eqref{q-pack} we have $z^{\top} M q \leq z^{\top} \one$.
Then \eqref{strong-duality} implies $\one^{\top}z = \one^{\top}p = q^{\top}M^{\top}z$.
Furthermore $p_v = (\diag{q}M^{\top}z)_v = q_v(M^{\top}z)_v$.

There is always some feasible $q$ that makes the objective function finite, so $q^*_v = 0$ implies $p_v = 0$.
For $q^*_v > 0$, the upper bound on $\log q_v$ from the linear approximation at $q^*_v$ is $\frac{q_v-q^*_v}{q^*_v} + \log q^*_v$.
Thus
\begin{multline*}
  \sum_{v : p_v > 0} p_v \log q_v
  \leq \sum_{v : p_v > 0} p_v \left(\frac{q_v-q^*_v}{q^*_v} + \log q^*_v\right)\\
  = \sum_{v : p_v > 0} \frac{p_v}{q^*_v} q_v - \one^{\top}p + \sum_{v : p_v > 0} p_v \log q^*_v. 
\end{multline*}

To prove $\sum_v -p_v \log q_v \geq \sum_v -p_v \log q^*_v$ for all $q$, we need $\sum_v \frac{p_v}{q^*_v} q_v \leq \one^{\top}p$.
To show this, we note that $z^{\top} M q \leq z^{\top} \one$ and $\one^{\top} z \leq \one^{\top} p$. Then, we only need that $(z^{\top}M)_v \geq \frac{p_v}{q^*_v}$, which follows from $\diag{q} M^{\top}z = p$.
\end{proof}

The vector $z$ in Lemma~\ref{lemma: convex-duality} can be interpreted as the optimal strategy followed by the adversary.

\subsection{Gaussian Case} \label{subsec: gauss_case}
We now consider the case when the data is generated from a mixture of two Gaussians with identical covariances and means that differ in their sign. Formally, we have $P=p_1\mathcal{N}(\mu, \Sigma)+p_{-1}\mathcal{N}(-\mu, \Sigma)$, where $p_1,p_{-1} \in [0,1]$ and $p_1+p_{-1}=1$. $\mcX$ is then $\mathbb{R}^d$. We set the neighborhood function $N(x)=x+\epsilon \Delta$, where $\epsilon$ is the adversarial budget and $\Delta \in \mathbb{R}^d$ is a closed, convex, absorbing and origin-symmetric set. 

Our first lemma proves that the optimal classifier is linear and the corresponding optimal adversarial strategy $z^*$ \footnote{We note that there is a slight abuse of notation here since $z$ in the previous section is a probability and is a perturbation here.} is just a translation of each component of the mixture. To show this, we just establish that these are identical to the solutions obtained in the $0-1$ loss case, allowing us to use Lemma 1 from \cite{bhagoji2019lower}.

\begin{lemma}
The optimal classifier $h^*_y$ minimizing the cross-entropy loss is given by $\frac{1}{1+\exp{(y(w^*)^\intercal x)}}$ where $w^*=2 \Sigma^{-1}(\mu-z^*)$, and $z^*$ is the optimal adversarial strategy given by Lemma 1 of \cite{bhagoji2019lower}.
\end{lemma}

The cross-entropy lower bound can then be directly computed.

\begin{theorem} The cross-entropy lower bound for a mixture of two Gaussians
is 
\begin{equation}
    \begin{aligned}
        &\inf_h \mathbb{E}_{P}[\tilde{\ell}^{\text{CE}}(h,v)] \\
    &=p_1\mathbb{E}_{N(\mu-z^*,\Sigma)}[\log(1+\exp{((w^*)^{\intercal}x}))]\\
    &+p_{-1}\mathbb{E}_{N(\mu+z^*,\Sigma)}[\log(1+\exp{-((w^*)^{\intercal}x}))]
    \end{aligned}
\end{equation}

\label{thm: ce_gaussian}
\end{theorem}
 
 We defer the proofs to Section \ref{appsec: proofs} of the Appendix.

\section{Efficiently Computing Lower Bounds}\label{sec: eff_comp}
In this section, we show that the convex program defined above can be efficiently solved by lower bounding its objective with a linear function and solving a recursive series of linear programs. We develop a specialized algorithm instead of using an off-the-shelf convex program solver in order to exploit the structure in the problem for faster computation.

\subsection{Algorithm overview}
Our algorithm (\textsf{OptProb}) executes the following strategy.
It starts by guessing that there is a single correct-classification probability that should be assigned to all vertices from class 1 and a single probability for vertices from class $-1$.
If this were the case, those probabilities should reflect the relative frequencies of the classes.
The algorithm solves a linear program and either finds a dual certificate proving that the initial guess is correct or a partition of the vertices based on whether the optimal correct-classification probabilities are larger or smaller than the guess.
In the latter case, the algorithm is applied recursively to the two subproblems and their solutions are assembled into a solution to the original problem.
A precise description appears as Algorithm~\ref{alg: opt_classifier_compute}.

\begin{algorithm}[t]
\caption{\textsf{OptProb}}
	\label{alg: opt_classifier_compute}
\begin{algorithmic}[1]
	\Require{Bipartite graph $(\mcA,\mcB, \mcE)$, vertex weights $P$}
	\Ensure{Classifier probabilities $q$, adversarial strategy $z$}
	\State $(\mcA^+,\mcA^-,\mcB^+,\mcB^-,z^{\text{lin}}) = \textsf{LinOpt}(\mcA,\mcB, \mcE,P)$
	\If{$P(\mcA^+)P(\mcB^+) > P(\mcA^-)P(\mcB^-)$}
	\State $\mcE' = \mcE\cap (\mcA^+ \times \mcB^-)$
	\State $\mcE'' = \mcE\cap (\mcA^- \times \mcB^+)$
	\State $(q',z') = \textsf{OptProb}(\mcA^+,\mcB^-,\mcE',P)$
	\State $(q'',z'') = \textsf{OptProb}(\mcA^-,\mcB^+,\mcE'',P)$
	\State $q = v \mapsto \begin{cases} 
	q'_v & v \in \mcA^+ \cup \mcB^-\\ 
	q''_v & v \in \mcA^- \cup \mcB^+ 
	\end{cases}$
	\State $z = e \mapsto \begin{cases} 
	z'_e & e \in (\mcA^+ \times \mcB^-) \cup \mcA^+ \cup \mcB^-\\ 
	z''_e & e \in (\mcA^- \times \mcB^+) \cup \mcA^- \cup \mcB^+\\
	0& \text{otherwise}
	\end{cases}$
	\Else 
	\State $q = v \mapsto \begin{cases} 
	P(\mcA)/P(\mcA \cup \mcB) & v \in \mcA\\ 
	P(\mcB)/P(\mcA \cup \mcB) & v \in \mcB 
	\end{cases}$ \label{q-base-case}
	\State $z = z^{\text{lin}}$
	\EndIf
	\State \Return $(q,z)$
	\end{algorithmic}
\end{algorithm}

The computation of $\textsf{OptProb}$ uses the function $\textsf{LinOpt}$ at each stage of the recursion.
The function $\textsf{LinOpt}(\mcA,\mcB,\mcE,P)$ solves a dual pair of linear programs with variables $y \in \R^{\mcA \cup \mcB}$ and $z \in \R^{(\mcA \times \mcB) \cup \mcA \cup \mcB}$ :
\begin{minipage}{.5\linewidth}
\begin{align*}
  \max r^{\top}y&\\
  y &\geq \zero\\
  M y &\leq \one
\end{align*}
\end{minipage}%
\begin{minipage}{.5\linewidth}
\begin{align*}
  \min \one^{\top}z&\\
  z &\geq \zero\\
  M^{\top}z &\geq r
\end{align*}
\end{minipage}
where $r \in \R^{\mcA \cup \mcB}$ is defined as follows.
If both $P(\mcA) > 0$ and $P(\mcB) >0$, then
\[
  r_v = \begin{cases} 
  P(\{v\})P(\mcA \cup \mcB)/P(\mcA) & v \in \mcA\\ 
  P(\{v\})P(\mcA \cup \mcB)/P(\mcB) & v \in \mcB 
  \end{cases}
\]
and otherwise $r_v = P(\{v\})$.

The primal polytope is the vertex packing polytope of the bipartite graph $(\mcA,\mcB,\mcE)$.
This is integral, so there is some optimal $y \in \{0,1\}^{\mcA \cup \mcB}$.
The sets $\mcA^+,\mcA^-,\mcB^+,\mcB^-$ encode the support of $y$ in a way that is convenient for expressing $\textsf{OptProb}$: $\mcA^+ = \{v \in \mcA : y_v = 1\}$, $\mcA^- = \mcA \setminus \mcA^+$, $\mcB^+ = \{v \in \mcB : y_v = 1\}$, and $\mcB^- = \mcB \setminus \mcB^+$.
The support of $y$ is an independent set: $(\mcA^+ \times \mcB^+) \cap \mcE = \varnothing$.

\subsection{Proof sketch for algorithm optimality}

\begin{theorem}[Convergence to optimal for algorithm]
The proposed Algorithm \ref{alg: opt_classifier_compute} returns the correct optimal classifier probability: the minimizer of \eqref{ce_lower_bound}.
\label{thm: alg_converge}
\end{theorem}
The proof of Theorem 3 mirrors the recursive structure of $\textsf{OptProb}$ and uses induction on the number of vertices. It relies on two technical lemmas (proofs deferred to Section \ref{appsec: alg_proofs} of the Appendix). Lemma 4 establishes properties of the solutions to linear programs that are solved at each iteration.
The proof uses standard duality and complementary slackness arguments for linear programs.

\begin{lemma}
  The function $\textsf{LinOpt}(\mcA,\mcB,\mcE,P)$ produces $(\mcA^+,\mcA^-,\mcB^+,\mcB^-,z)$ with the following properties.
  \begin{outline}[enumerate]
      \1 $P(\mcA^+)P(\mcB^+) \geq P(\mcA^-)P(\mcB^-)$
      \1 If $P(\mcA^+)P(\mcB^+) = P(\mcA^-)P(\mcB^-)$, then
      \2 $\one^{\top}z = P(\mcA \cup \mcB)$, 
      \2$\frac{P(\mcA)}{P(\mcA \cup \mcB)}(M^{\top} z)_v = P(\{v\})$ for all $v \in \mcA$, 
      \2$\frac{P(\mcB)}{P(\mcA \cup \mcB)}(M^{\top} z)_v = P(\{v\})$ for all $v \in \mcB$.
  \end{outline}
\label{lemma: linopt-props}
\end{lemma}

Lemma 5 establishes that $\textsf{OptProb}$ terminates and describes the structure of the optimal $(q,z)$ in detail.
Both vertex sets are split in partitions with paired parts and each pair of parts behaves similarly to a complete bipartite graph. Edges between pairs of parts are restricted by an order relation. 
\begin{lemma}
\label{lemma: optprob-props}
If $P(\mcA \cup \mcB) > 0$, the computation of $\textsf{OptProb}(\mcA,\mcB,\mcE,P)$ terminates and produces a pair $(q,z)$.
For some $[k] = \{0,1,\ldots,k-1\}$, there are functions $a : \mcA \to [k]$ and $b: \mcB \to [k]$ with the following properties.
\begin{outline}[enumerate]
\1 If $(u,v) \in \mcE$, $a(u) \leq b(v)$.
\1 We have $z \geq 0$, $\one^{\top}z = P(\mcA \cup \mcB)$, and $P(\{v\}) = q_v (M^{\top}z)_v$.
\1 Let $\mcA_{i} = a^{-1}(i)$ and $\mcB_{i} = b^{-1}(i)$.\\
For all $i$, $P(\mcA_{i} \cup \mcB_{i}) > 0$.
For all $u \in \mcA$ and $v \in \mcB$, 
$q_u = \frac{P(\mcA_{a(u)})}{P(\mcA_{a(u)} \cup \mcB_{a(u)})}$ and
$q_v = \frac{P(\mcA_{b(v)})}{P(\mcA_{b(v)} \cup \mcB_{b(v)})}$.

\1 For $u,u' \in \mcA$, if $a(u) \leq a(u')$ then $q_u \leq q_{u'}$.
\end{outline}
\end{lemma}

\begin{proof}[Proof of Theorem~\ref{thm: alg_converge}]
From Lemma~\ref{lemma: optprob-props}, we have that the computation of $\textsf{OptProb}$ terminates and some information about $(q,z) = \textsf{OptProb}(\mcV_1,\mcV_{-1},\mcE,P)$.
Properties 1,3, and 4 together imply \eqref{q-pos} and \eqref{q-pack} (i.e. that $q$ is feasible in \eqref{ce_lower_bound}): for any $(u,v) \in \mcE$, there is some $u'$ such that $a(u') = b(v)$ and $q_u \leq q_{u'} = 1 - q_v$.
Property 2 provides \eqref{z-pos}, \eqref{z-cover}, and \eqref{strong-duality}.
Lemma~\ref{lemma: convex-duality} establishes the optimality of $q$.
\end{proof}

\subsection{Complexity analysis}
In the worst case, at each step of the algorithm, only a single vertex will be removed from one part of the bipartite graph, and the algorithm will only terminate when only singleton parts of the graph remain. In this case, if there are $|\mathcal{V}|$ vertices in the graph, there will be $|\mathcal{V}|$ recursive steps, with each run taking $O(|\mathcal{V}||\mathcal{E}|\log(|\mathcal{V}|^2/|\mathcal{E}|))$ \cite{goldberg1988new}.

\section{Experiments: Using Bounds as a Diagnostic Tool}\label{sec: training_improve}
\begin{figure*}[t]
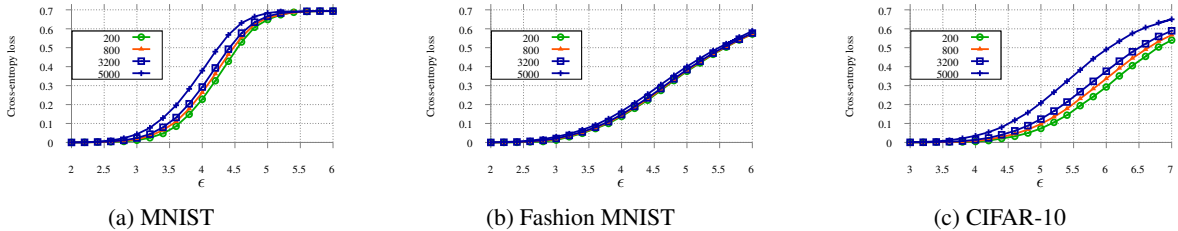

	\centering
	\subfloat[MNIST]{\resizebox{0.32\textwidth}{!}{\input{plots/3_7_mnist_l2_logloss_subsample_.tex}}\label{subfig: mnist_sub}}
	\hspace{0mm}
	\subfloat[Fashion MNIST]{\resizebox{0.32\textwidth}{!}{\input{plots/3_7_fmnist_l2_logloss_subsample_.tex}}\label{subfig: fmnist_sub}}
	\hspace{0mm}
	\subfloat[CIFAR-10]{\resizebox{0.32\textwidth}{!}{\input{plots/3_7_cifar_l2_logloss_subsample_.tex}}\label{subfig: cifar_sub}}
	\caption{Variation in minimum log-loss for an $\ell_2$ adversary with adversarial budget $\epsilon$ and the number of samples from each class. The maximum possible log-loss is $\ln 2$, which is around $0.693$. The total number of samples is $5000$.}
	\label{fig: subsample}
	\vspace{-10pt}
\end{figure*}

In the previous section, we derived lower bounds on the cross-entropy loss that are applicable for all discrete distributions as well as for Gaussian data. In this section, we compute and use these bounds as a diagnostic tool to better understand limits of robust learning for practical datasets and algorithms. We determine lower bounds on the cross-entropy loss for practical datasets of interest. We analyze the runtime of Algorithm \ref{alg: opt_classifier_compute} and show its speedup over the generic non-linear convex solver from CVXOPT \cite{andersen2013cvxopt}. Finally, we uncover a gap between the loss obtained by several robust training methods and the lower bound, and investigate the use of `soft-label' training with optimal classifier outputs to close this gap. All results are obtained on an Intel Xeon cluster with 8 P100 GPUs.

\subsection{Lower bounds on robustness for real-world datasets}
From Theorem \ref{thm: alg_converge} and Algorithm \ref{alg: opt_classifier_compute}, we have an efficient method to compute the optimal log-loss for any empirical distribution. Here, we consider 3 benchmark computer vision datasets: MNIST \cite{lecun1998mnist}, Fashion MNIST \cite{xiao2017/online} and CIFAR-10 \cite{krizhevsky2009learning}. Each of these datasets is originally a 10-class classification problem, and from each, without loss of generality, we choose the `$3$ vs. $7$' classification task as a representative binary classification problem (results for other choices are in Section \ref{appsec: extra_results} of the Appendix). In each case, there are a total of $n=5000$ training samples per class which can be used to compute the lower bound.

To derive a numerical bound, we need to specify the neighborhood function (adversarial constraints). While our bounds are valid for any non-empty neighborhood function, we pick the commonly used $\ell_2$-norm ball constraint, parametrized by its radius $\epsilon$. This has been used numerous times for both attacks \cite{carlini2017towards} and defenses \cite{madry_towards_2017}, and has well-established benchmarks \cite{advbench}. Although $\ell_p$-norm constraints have been critiqued \cite{gilmer2018motivating,evtimov2020security}, we nonetheless choose to use them to provide a point of comparison with existing work.

\noindent \textbf{Algorithm implementation:} We first create the conflict graph by checking for $\ell_2$ ball intersections between all pairs of points from the two classes. The number of vertices $\mcV$ in the conflict graph $\mcG$ is $n_{1}+n_{-1}$. We will generally consider the case when the total number of datapoints in each class is equal, giving $|\mcV|=2n$. The total number of edges $\mcE$ is then $\hat{p}(n,\epsilon)n^2$, where $\hat{p}(n,\epsilon)$ is an estimate of the probability that the neighborhoods around points from the two classes have a non-empty intersection. We find that for the $\ell_2$ norm, $\hat{p}(n,\epsilon)$ increases monotonically with $\epsilon$ and unlike the log-loss, is largely independent of the number of samples (Section \ref{appsec: extra_results} of the Appendix).

Using this conflict graph, represented by a sparse matrix, we use Algorithm \ref{alg: opt_classifier_compute} to compute the lower bound. We use the maximum flow algorithm from Scipy \cite{2020SciPy-NMeth} as \textsf{LinOpt} at the top level and for each recursively obtained split. This implementation uses the Edmonds-Karp \cite{edmonds1972theoretical} algorithm. We note that any linear program solver can be used and casting the problem as maximum flow is not canonical.

\noindent \textbf{Numerical lower bounds}: In Figure \ref{fig: subsample}, we plot the variation in the minimum cross-entropy loss over the full set of $5000$ training samples for all $3$ datasets as the adversary's $\ell_2$ budget is varied. The lower bound is only non-trivial after a budget of around $3.0$ for the MNIST dataset and $4.0$ for the CIFAR-10 dataset. At smaller budgets, the optimal classifier can achieve $0$ loss even in the presence of an adversary. We note that this classifier may not generalize well to test data, since these bounds do not represent the population lower bound over the unknown underlying distribution.


\noindent \textbf{Impact of subsampling}: We also analyze the impact of subsampling from the complete set of samples to understand the dependence of the lower bound on the number of samples. We find that as the number of samples increases, the lower bound increases as well, indicating the presence of more intersections among samples, and thus more flexibility for the adversary.

\begin{figure}[t]
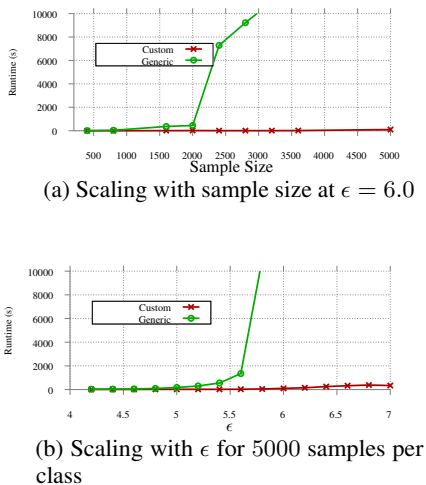

	\centering
	\subfloat[Scaling with sample size at $\epsilon=6.0$]{\resizebox{0.3\textwidth}{!}{	\input{plots/cifar_sample_var_}\label{subfig: time_vs_sample}}}
	\hspace{1pt}
	\subfloat[Scaling with $\epsilon$ for $5000$ samples per class]{\resizebox{0.3\textwidth}{!}{\input{plots/cifar_eps_var_}}\label{subfig: time_vs_eps}}
	\caption{Algorithm runtime comparisons for CIFAR-10}
	\label{fig: timing}
	\vspace{-15pt}
\end{figure}

\noindent \textbf{Empirical runtime comparison:} We compare the runtime of Algorithm~\ref{alg: opt_classifier_compute} using the max-flow solver from Scipy to that of the general purpose solver for convex programs with non-linear objective functions from CVXOPT \cite{andersen2013cvxopt}, which uses primal-dual interior point methods \cite{boyd2004convex}. 

The two parameters that determine the runtime of the algorithms to compute the minimum log-loss are the number of vertices $|\mcV|$ and the adversary's budget $\epsilon$ which controls the graph density. In Figure \ref{subfig: time_vs_sample}, we show the variation in CPU time in seconds as the number of vertices in each class is varied. The mean and standard deviation over 10 runs is reported and the maximum time either algorithm is allowed to run is $10,000$ seconds after which it is terminated. It is clear that our custom algorithm runs significantly faster than the general purpose convex solver, with speed-ups of up to $3000 \times$. The advantages are even starker as $\epsilon$ is varied in Figure \ref{subfig: time_vs_eps}, with the general purpose solver taking in excess of $10,000$ seconds for any budget greater than $5.6$.
We can draw the same conclusions for the other two datasets from the runtime analysis presented in Section \ref{appsec: extra_results} of the Appendix. 

\subsection{Synthetic Gaussian data}
From Section \ref{subsec: gauss_case}, we have a complete characterization of the robust learning problem for Gaussian distributions with respect to the cross-entropy loss. We use this to study the gap between population-level and sample-level cross-entropy lower bounds, finding that \emph{this gap increases with the dimension of the data.} Thus, when the underlying distribution is unknown, sample-level lower bounds must be used carefully, especially with a small number of samples.

We use a diagonal covariance matrix $\Sigma$ with $\Sigma_{ii}$ sampled uniformly between $0$ and $1$, and set $\mu_i=C*\frac{\Sigma_{ii}}{\sqrt{d}}$, where $C$ is a constant determining the distance between the means. The two classes have identical covariances and means of opposite sign. In Figure \ref{fig:syn_gauss_plot}, we compare the lower bound on cross-entropy loss directly obtained from Theorem \ref{thm: ce_gaussian} (`Population loss') and that over the empirical distributions resulting from sampling it ('$k$ samples') for $d=100$. In the latter case, the lower bounds are computed using Algorithm \ref{alg: opt_classifier_compute}. The reason for the lack of intersections at lower budgets for the empirical distribution is that in high dimensions, even when the underlying distributions overlap, further perturbation is needed for intersections between the neighborhoods of sampled points. Results for other choices of $d$ are in Section \ref{appsec: extra_results} of the Appendix.

%

\subsection{Evaluating the performance of robust training}
We now compare the cross-entropy loss obtained by robust training techniques such as adversarial training \cite{madry_towards_2017} and TRADES \cite{zhang2019theoretically} to our lower bounds. We present results on the MNIST \cite{lecun1998mnist} and Fashion-MNIST \cite{xiao2017/online} datasets in the main body, and on CIFAR-10 in Section \ref{appsec: robust_train_extra} of the Appendix.

Our \emph{key takeaways} are i) standard adversarial training can achieve close to the minimum cross-entropy loss with a sufficiently large architecture, but a gap still remains for the $0-1$ loss and, ii) soft label training with optimal probabilities obtained from our framework can help close this gap as well as aid in generalization in some cases.

\begin{figure}
    \centering
    \input{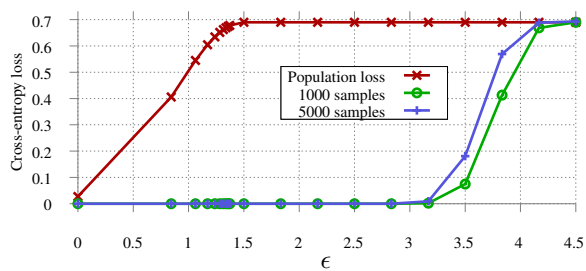}
    \caption{Comparing the population-level and sample-level lower bounds on cross-entropy loss for synthetic $2$-class Gaussian data of dimension $100$.}
    \label{fig:syn_gauss_plot}
    \vspace{-10pt}
\end{figure}



\begin{figure}[t]
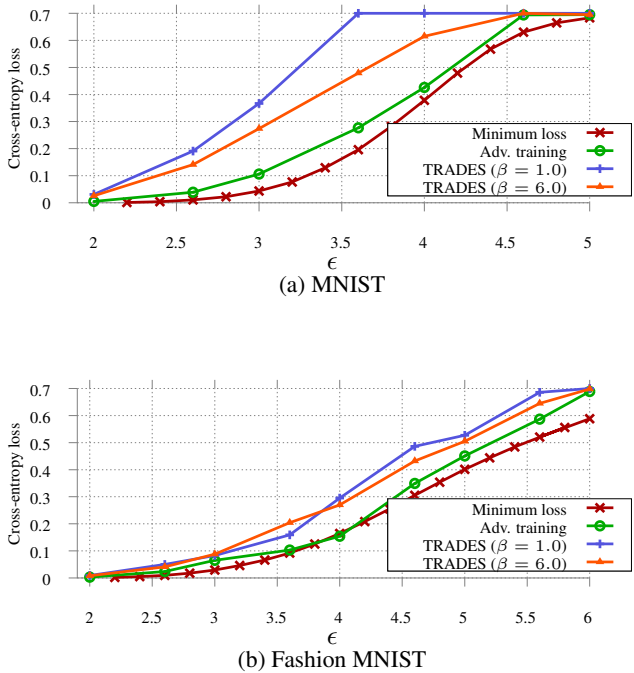

	\centering
	\subfloat[MNIST]{\resizebox{0.48\textwidth}{!}{\input{plots/3_7_mnist_l2_5000_loss_compare_resnet_}}\label{subfig: mnist_KL}}
	\hspace{0mm}
	\subfloat[Fashion MNIST]{\resizebox{0.48\textwidth}{!}{\input{plots/3_7_fmnist_l2_5000_loss_compare_resnet_}}\label{subfig: fmnist_KL}}
	\caption{Comparison on \emph{training data} between the cross-entropy loss (computed using AutoAttack) obtained by different training methods versus the optimal loss.}
	\label{fig: compare_plot_KL}
	\vspace{-15pt}
\end{figure}

\noindent \textbf{Robust training setup.} We train a ResNet-18 network using adversarial training and TRADES, these being the most effective robust training methods for an $\ell_2$ adversary \cite{croce2020robustbench}. The robust cross-entropy loss for these models is computed using the state-of-the-art AutoAttack \cite{croce2020reliable}. Adversarial training, referred to as `hard labels' in Figure~\ref{fig: compare_plot_KL}, utilizes one-hot labels, while TRADES uses the network's own prediction as soft-labels.

\noindent \textbf{How close is the robust training loss of current techniques to optimal?} In Figure \ref{fig: compare_plot_KL}, for both datasets, adversarial training achieves close to the minimum possible cross-entropy loss on the training data. However, TRADES is outperformed by standard adversarial training with hard labels. This runs counter to earlier observations at lower adversarial budgets that TRADES was more robust. In the case of the $0-1$ loss for all datasets and the cross-entropy loss for CIFAR-10, the gap is far larger even for moderate budgets (Section \ref{appsec: robust_train_extra} of the Appendix). Nevertheless, since a gap exists even for the cross-entropy loss, we can rule out the possibility that the gap previously observed for the $0-1$ loss in \citet{bhagoji2019lower} is only due to the use of a surrogate loss.

We also conduct ablation studies with larger networks and smoother activation functions, techniques known to help with robust training. Resnet-101 for FMNIST reduces cross-entropy loss to $0.42$, in comparison to $0.45$ with ResNet-18, which is close to the optimal loss for Fashion MNIST at $\epsilon=5.0$. Additionally, with over $15$ different activation functions, we did not observe any significant drop in cross-entropy loss compared to the standard ReLU. Further details and results are in Section \ref{appsec: robust_train_extra} of the Appendix.

\noindent \textbf{Using soft labels.} Training using soft labels is known to improve the performance of deep neural networks~\cite{zheng2016stability}. Since at higher values of $\epsilon$, the optimal classifier may assign a higher probability to the opposite class as the true label, the obtained soft labels are noisy. To avoid introducing this label noise, while also extracting meaningful gradients, we impose a lower bound on the probability of the correct class (details in Section \ref{appsec: robust_train_extra} of the Appendix). We find that training with these clipped soft-labels can reduce the cross-entropy loss by a significant margin (Table \ref{tab:clip_soft_labels}). Additionally, this method can also improve the $0-1$ loss for the MNIST datasets for a range of budgets (Section \ref{appsec: robust_train_extra} of the Appendix). Overall, these results indicate that appropriately calibrated soft label training can help with robustness.

\section{Related Work}\label{sec: rel_work}

We only discuss the closest related work here on theoretical analysis of test-time adversaries and robust training. Extensive surveys \cite{papernot2016towards,liu2018survey,biggio2017wild,li2020sok} provide a broader overview.

\begin{table}[t]
\centering
\caption{Comparison of train and test set robust accuracy with different robust training techniques for the FMNIST dataset.}
\label{tab:clip_soft_labels}
\resizebox{\linewidth}{!}{
\begin{tabular}{ccccccc}
\toprule
 &  & \multicolumn{2}{c}{FMNIST ($\epsilon=4.6$)} &  & \multicolumn{2}{c}{FMNIST ($\epsilon=5.0$)} \\ \cmidrule{3-4} \cmidrule{6-7}
 &  & Train & Test &  & Train & Test \\ \midrule
Hard labels &  & $0.349$ & $0.348$ &  & $0.451$ & $0.451$ \\
Clipped soft labels &  & $0.326$ & $0.331$ &  & $0.419$ & $0.420$ \\
Optimal &  & 0.305 & -- &  & 0.401 & -- \\ \bottomrule
\end{tabular}}
\vspace{-10pt}
\end{table}

\noindent \textbf{Information-theoretic limits on robust learning.} All previous work on information-theoretic limits on robust learning has focused on the $0-1$ loss. \cite{pmlr-v97-dohmatob19a} and \cite{mahloujifar2019curse} use the `blowup' property of specific data distributions to determine bounds on the robust loss, given some level of loss on benign data. \cite{bhagoji2019lower} and \cite{pmlr-v119-pydi20a} use optimal transport to provide lower bounds on the robust loss for a general class of distributions, without a dependence on the loss on benign data. While \cite{pmlr-v119-pydi20a} does consider convex losses, we are the first to provide an explicit method and framework, as well as numerical results, for the cross-entropy loss.

\noindent \textbf{Generalization for adversarially robust learning.} A number of papers analyze the sample complexity of robust learning for specific distributions of interest such as Gaussians \cite{schmidt2018adversarially,pmlr-v125-javanmard20a,pmlr-v119-dan20b}, uniform \cite{diochnos2018adversarial} and spherical \cite{gilmer2018adversarial}. The sample complexity of PAC-learning (worst case over distributions) for robust classifiers has also been derived \cite{cullina2018pac,yin2018rademacher,montasser2019vc}. However, this line of work does not analyze the minimum possible loss, only the gap between the minimum and learned.

\noindent \textbf{Computational limits of robust learning.} Computationally bounded adversaries \cite{pmlr-v117-garg20a} were considered to devise instances where there is a separation between their power and that of unbounded adversaries. Other work \cite{bubeck2018adversarial,awasthi2019robustness, pmlr-v119-montasser20a} has focused on instances where computationally efficient robust learning is possible.

\noindent \textbf{Robust training of neural networks.} Adversarial training~\cite{madry_towards_2017} with follow-up improvements in TRADES~\cite{zhang2019theoretically}, remains the most successful robust training technique. Its performance is further improved with larger networks~\cite{gowal2020uncovering}, smooth activations~\cite{xie2020smoothadv}, early stopping~\cite{rice2020overfitting}, and careful tuning of weight decay~\cite{pang2021bagoftricks}. Some other works investigate the effect of weight perturbation~\cite{wu2020advweightperturb}, weight averaging~\cite{gowal2020uncovering}, sub-networks on robustness~\cite{sehwag2020hydra} and additional data \cite{carmon2019unlabeled}. \citet{Wang2020revisit} further demonstrate minor improvements in robustness with sample-weighted adversarial training. \citet{goibert2019adversarial} show that the use of smoothed labels can aid with robustness. However, we note that unlike this paper, they do not use the labels from the optimal classifier. For a detailed comparison of state-of-the-art robust training techniques, we refer the reader to RobustBench~\cite{croce2020robustbench}. 


\section{Discussion}\label{sec: discussion}
In this paper, we have provided a framework to compute optimal lower bounds on the cross-entropy loss for general discrete distributions as well as Gaussian mixtures. We showed how to leverage this framework to analyze current robust training methods. In future work on the theoretical front, we plan to extend our framework to all continuous probability distributions as well as the multi-class case. On the empirical front, we aim to further investigate the convergence of robust training for complex datasets such as CIFAR-10, as well as to use our framework to guide the generation of more robust feature representations. 



\section*{Acknowledgements}
This work was supported in part by the National Science Foundation under grants CNS-1553437, CNS-1704105 and CNS-1949650, the DARPA GARD program, the ARL’s Army Artificial Intelligence Innovation Institute (A2I2), the Office of Naval Research Young Investigator Award, the Army Research Office Young Investigator Prize, a faculty research award from Facebook, the Schmidt DataX award, and Princeton E-ffiliates Award. Any opinions, findings, and conclusions or recommendations expressed in this material are those of the authors and do not
necessarily reflect the views of any funding agencies.

\bibliography{main_arxiv}
\bibliographystyle{icml2020}

\clearpage

 \appendix

\section{Proof for Lemma 3 and Theorem 2} \label{appsec: proofs}
We consider the case when the data is generated from a mixture of two Gaussians with identical covariances and means that differ in their sign. Formally, we have $P_Y(1) = p_1$, $P_Y(-1) = p_{-1}$, and $P_{X|Y=y}=\mathcal{N}(y\mu, \Sigma)$.
$\mcX$ is then $\mathbb{R}^d$. We set the neighborhood function $N(x)=x+\epsilon \Delta$, where $\epsilon$ is the adversarial budget and $\Delta \in \mathbb{R}^d$ is a closed, convex, absorbing and origin-symmetric set. 




\begin{proof}
Let $c = \log\frac{p_1}{p_{-1}}$ so $yc = \log\frac{p_y}{p_{-y}}$.
Let $w \in \R^d$ and consider the classifier 
\begin{align*}
  h(x)_y &= \frac{1}{1+\exp(-y (w^{\top}x + c))}\\ 
  &= \frac{p_y\exp(\frac{y}{2} w^{\top}x)}
  {p_y\exp(\frac{y}{2} w^{\top}x)+p_{-y}\exp(\frac{-y}{2} w^{\top}x)}.
\end{align*}
The output probability $h(x)_y$ is an increasing function of $yw^{\top}x$, so we can find $q_{(x,y)}=\inf_{\tilde{x} \in N(x)} h(\tilde{x})_y$ by computing $\inf_{\tilde{x} \in N(x)} yw^{\top}\tilde{x} = yw^{\top}x - \sup_{z \in \epsilon\Delta} w^{\top}z = yw^{\top}x - \epsilon\|w\|_{\Delta}^*$.
Thus adversarial log loss of this classifier is
\[
\sum_y p_y \E_{X \sim \mathcal{N}(y\mu,\Sigma)} \log (1+\exp(-y (w^{\top}X + c) + \epsilon\|w\|_{\Delta}^*))
\]
where $X \sim \mathcal{N}(y\mu,\Sigma)$ and this is an upper bound on the optimal adversarial log loss.
Observe that \[
yw^{\top}X - \epsilon\|w\|_{\Delta}^* \sim \mathcal{N}(w^{\top}\mu - \epsilon\|w\|_{\Delta}^*,w^{\top}\Sigma w).
\]

For any $z \in \Delta$, the distributions $P_{\tilde{X}|Y=y}=\mathcal{N}(y(\mu-z), \Sigma)$ are clearly feasible for the adversary.
The Bayes classifier for these is 
\begin{align*}
  h(x)_y 
  &= \frac{1}{1+\exp(-y (2(\mu-z)^{\top}\Sigma^{-1}x + c))}.
\end{align*}
The log loss of this classifier is
\[
\sum_y p_y \E \log (1+\exp(-y (2(\mu-z)\Sigma^{-1}X + c)))
\]  
where $X \sim \mathcal{N}(y(\mu-z),\Sigma)$ and this is an lower bound on the optimal adversarial log loss.
Observe that 
\begin{multline*}
2y(\mu-z)^{\top}\Sigma^{-1}X \sim\\ \mathcal{N}(2(\mu-z)^{\top}\Sigma^{-1}(\mu-z),4(\mu-z)^{\top}\Sigma^{-1}(\mu-z)).
\end{multline*}
If we can find $w$ and $z$ such that
\begin{align*}
    w^{\top}\mu - \epsilon\|w\|_{\Delta}^* &= 2(\mu-z)^{\top}\Sigma^{-1}(\mu-z)\\
    w^{\top}\Sigma w &= 4(\mu-z)^{\top}\Sigma^{-1}(\mu-z),
\end{align*}
then these upper and lower bounds match.

Using Lemma 1 from \cite{bhagoji2019lower},
if we take $z$ to be the solution to optimization problem
\[
    \min (\mu-z)^{\top}\Sigma^{-1}(\mu-z) \text{ s.t. } z \in \epsilon\Delta
\]
and $w = 2\Sigma^{-1}(\mu-z)$,
then $\epsilon\|w\|_{\Delta}^* = w^{\top} z$, which immediately implies the desired equalities.
\end{proof}

\section{Proofs for Algorithm 1}\label{appsec: alg_proofs}

\subsection{Proof of Lemma 4}
\begin{proof}
Because each edge contains exactly one vertex in each of $\mcA$ and $\mcB$, $M \one_{\mcA} = \one_{\mcE \cup \mcA}$ and $M \one_{\mcB} = \one_{\mcE \cup \mcB}$.
This gives two feasible choices for $y$: $y = \one_{\mcA}$ and $y = \one_{\mcB}$.
By construction of $r$, at least one of these achieves a value of $P(\mcA \cup \mcB)$.
If $P(\mcA)>0$ then
\[
  r^{\top}\one_{\mcA} = \sum_{v \in \mcA} \frac{P(\mcA \cup \mcB)}{P(\mcA)} p_v =  P(\mcA \cup \mcB)
\]
and if $P(\mcB)>0$ then $r^{\top}\one_{\mcB} = P(\mcA \cup \mcB)$.
Complementary slackness implies $(M^{\top}z - r)^{\top}y = 0$ for all optimal $y$,
and thus $(M^{\top}z - r)_v = 0$ for all $v$ that are nonzero in some optimal $y$.
If the feasible points $y = \one_{\mcA}$ and $y = \one_{\mcB}$ are optimal, then $(M^{\top} z)_v = r_v$ for all $v \in \mcA$ if $P(\mcA) > 0$ and for all $v \in \mcB$ if $P(\mcB) > 0$.
Then Property 2 follows from the definition of $r$.

By strong linear programming duality, we always find $z$ and $y$ such that $\one^{\top} z = r^{\top} y$.
If the candidate choices of $y$ described above are optimal, we satisfy the first alternative of the claim.
Otherwise, we have $y$ such that $r^{\top} y > \one^{\top} p$.
We have
\begin{align*}
  r^{\top}(\one_{\mcA_+} + \one_{\mcB_+}) &> \one^{\top} p\\ 
  \frac{P(\mcA \cup \mcB)P(\mcA^+)}{P(\mcA)} + \frac{P(\mcA \cup \mcB)P(\mcB^+)}{P(\mcB)} &>  P(\mcA \cup \mcB)\\
  P(\mcA^+)P(\mcB)+P(\mcB^+)P(\mcA) &> P(\mcA)P(\mcB)\\
  P(\mcA^+)P(\mcB^+) &> P(\mcA^-)P(\mcB^-)
\end{align*}
which establishes Property 1.
\end{proof}

\subsection{Proof of Lemma 5}



\begin{proof}
In the base case of the induction, the output of $\textsf{OptProb}$ comes from the second branch, the computation terminates, and $P(\mcA^+)P(\mcB^+) \leq P(\mcA^-)P(\mcB^-)$.
We take $k=1$, so $\mcA = \mcA_0$ and $\mcB = \mcB_0$.
From Property 1 of Lemma 4, $P(\mcA^+)P(\mcB^+) = P(\mcA^-)P(\mcB^-)$ and thus from Property 2a we have $\one^{\top}z = P(\mcA \cup \mcB)$
Then $q$ is specified by Line \ref{q-base-case} and satisfies Property 3 by construction.
Properties 2b and 2c of Lemma 4 implies $q_v(M^{\top} z)_v = P(\{v\})$, so Property 2 is established.
Properties 1 and 4 hold trivially when $k=1$.

In the inductive case, the output of $\textsf{OptProb}$ comes from the first branch.
Because $P(\mcA^+)P(\mcB^+) > 0$, both $\mcA^+$ and $\mcB^+$ are nonempty.
Thus $|\mcA^+ \cup \mcB^-| < |\mcA \cup \mcB|$ and $|\mcA^- \cup \mcB^+| < |\mcA \cup \mcB|$, so the recursive calls both involve strictly smaller vertex sets.
By induction, both recursive calls terminate.
Suppose that $(q',z')$ and $(q'',z'')$ satisfy the four properties with functions $a' : \mcA^+ \to [k']$ and $b' : \mcB^- \to [k']$ and $a'' : \mcA^- \to [k'']$ and $b'' : \mcB^+ \to [k'']$ respectively.
Then we take $k = k'+k''$ and define $a$ and $b$ in the following piecewise fashion:
\begin{align*}
a(u) &= \begin{cases}
a'(u)+k'' & u \in \mcA^+\\
a''(u) & u \in \mcA^-\\
\end{cases}\\
b(u) &= \begin{cases}
b'(u)+k'' & u \in \mcB^-\\
b''(u) & u \in \mcB^+\\
\end{cases}
\end{align*}
Because $\mcA^+ \cup \mcB^+$ is an independent set, there are no edges $(u,v)$ with $a(u) \geq k'' > b(v)$.
Along with the induction hypotheses, this established Property 1.
The piecewise definitions of $q$ in line 7 and $z$ in line 8 satisfy Properties 2 and 3 because $(q',z')$ and $(q'',z'')$ do.

Property 4 requires a bit of calculation.
The set $\mcA^+ \cup \mcA_{k''-1} \cup (\mcB^+ \setminus \mcB_{k''-1})$ is an independent set and from the properties of $\textsf{LinOpt}$
\begin{multline*}
P(\mcB)P(\mcA^+) + P(\mcA)P(\mcB^+) \geq\\ P(\mcB)(P(\mcA^+) + P(\mcA_{k''-1})) + P(\mcA)(P(\mcB^+) - P(\mcB_{k''-1}))
\end{multline*}  
so $P(\mcA)P(\mcB_{k''-1}) \geq P(\mcB)P(\mcA_{k''-1})$.
Similarly, $(\mcA^+ \setminus \mcA_{k''}) \cup \mcB^+ \cup \mcB_{k''}$ is an independent set and 
\begin{multline*}
P(\mcB)P(\mcA^+) + P(\mcA)P(\mcB^+) \geq\\ P(\mcB)(P(\mcA^+) - P(\mcA_{k''})) + P(\mcA)(P(\mcB^+) + P(\mcB_{k''}))
\end{multline*}  
so $P(\mcA)P(\mcB_{k''}) \leq P(\mcB)P(\mcA_{k''})$.
Combining these inequalities, we have 
\[
\frac{P(\mcA_{k''-1})}{P(\mcA_{k''-1} \cup \mcB_{k''-1})} \leq 
\frac{P(\mcA)}{P(\mcA \cup \mcB)} \leq
\frac{P(\mcA_{k''})}{P(\mcA_{k''} \cup \mcB_{k''})}.
\]
Along with the induction hypotheses, this establishes Property 4.
\end{proof}

\section{Additional Results}\label{appsec: extra_results}
In this section we present additional results that were omitted from the main body of the paper for space considerations.

\subsection{Other class pairs}
In Figures \ref{suppfig: subsample_1_9} and \ref{suppfig: subsample_2_8}, we present the results for the lower bound on cross-entropy loss for two other choices of class pairs, `1 vs. 9' and '2 vs. 8'. We can see that while the exact values of the lower bound differ, the trend with respect to both the adversarial budget and the number of samples is the same as in the `3 vs. 7' case.

\begin{figure*}[t]
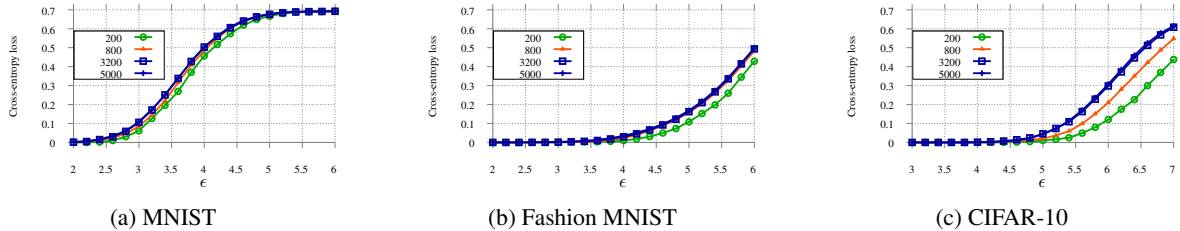

	\centering
	\subfloat[MNIST]{\resizebox{0.32\textwidth}{!}{\input{plots/1_9_mnist_l2_logloss_subsample_.tex}}\label{subfig: mnist_sub_1_9}}
	\hspace{0mm}
	\subfloat[Fashion MNIST]{\resizebox{0.32\textwidth}{!}{\input{plots/1_9_fmnist_l2_logloss_subsample_.tex}}\label{subfig: fmnist_sub_1_9}}
	\hspace{0mm}
	\subfloat[CIFAR-10]{\resizebox{0.32\textwidth}{!}{\input{plots/1_9_cifar_l2_logloss_subsample_.tex}}\label{subfig: cifar_sub_1_9}}
	\caption{\textbf{Two class problem is `1 vs. 9'}. Variation in minimum log-loss for an $\ell_2$ adversary with adversarial budget $\epsilon$ and the number of samples from each class. The maximum possible log-loss is $\ln 2$, which is around $0.693$. The total number of samples is $5000$.}
	\label{suppfig: subsample_1_9}
	\vspace{-10pt}
\end{figure*}

\begin{figure*}[t]
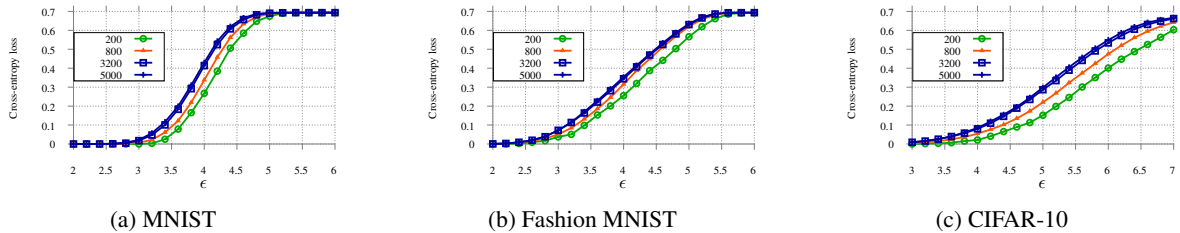

	\centering
	\subfloat[MNIST]{\resizebox{0.32\textwidth}{!}{\input{plots/2_8_mnist_l2_logloss_subsample_.tex}}\label{subfig: mnist_sub_2_8}}
	\hspace{0mm}
	\subfloat[Fashion MNIST]{\resizebox{0.32\textwidth}{!}{\input{plots/2_8_fmnist_l2_logloss_subsample_.tex}}\label{subfig: fmnist_sub_2_8}}
	\hspace{0mm}
	\subfloat[CIFAR-10]{\resizebox{0.32\textwidth}{!}{\input{plots/2_8_cifar_l2_logloss_subsample_.tex}}\label{subfig: cifar_sub_2_8}}
	\caption{\textbf{Two class problem is `2 vs. 8'}. Variation in minimum log-loss for an $\ell_2$ adversary with adversarial budget $\epsilon$ and the number of samples from each class. The maximum possible log-loss is $\ln 2$, which is around $0.693$. The total number of samples is $5000$.}
	\label{suppfig: subsample_2_8}
	\vspace{-10pt}
\end{figure*}

\subsection{Graph properties}
We show the variation in collision probability with the budget for different numbers of samples per class in Figure \ref{fig:cifar-10_collide}. This quantity can be estimated accurately even with a small number of samples, unlike the lower bound on cross-entropy.

\begin{figure}
    \centering
    \includegraphics[width=\columnwidth]{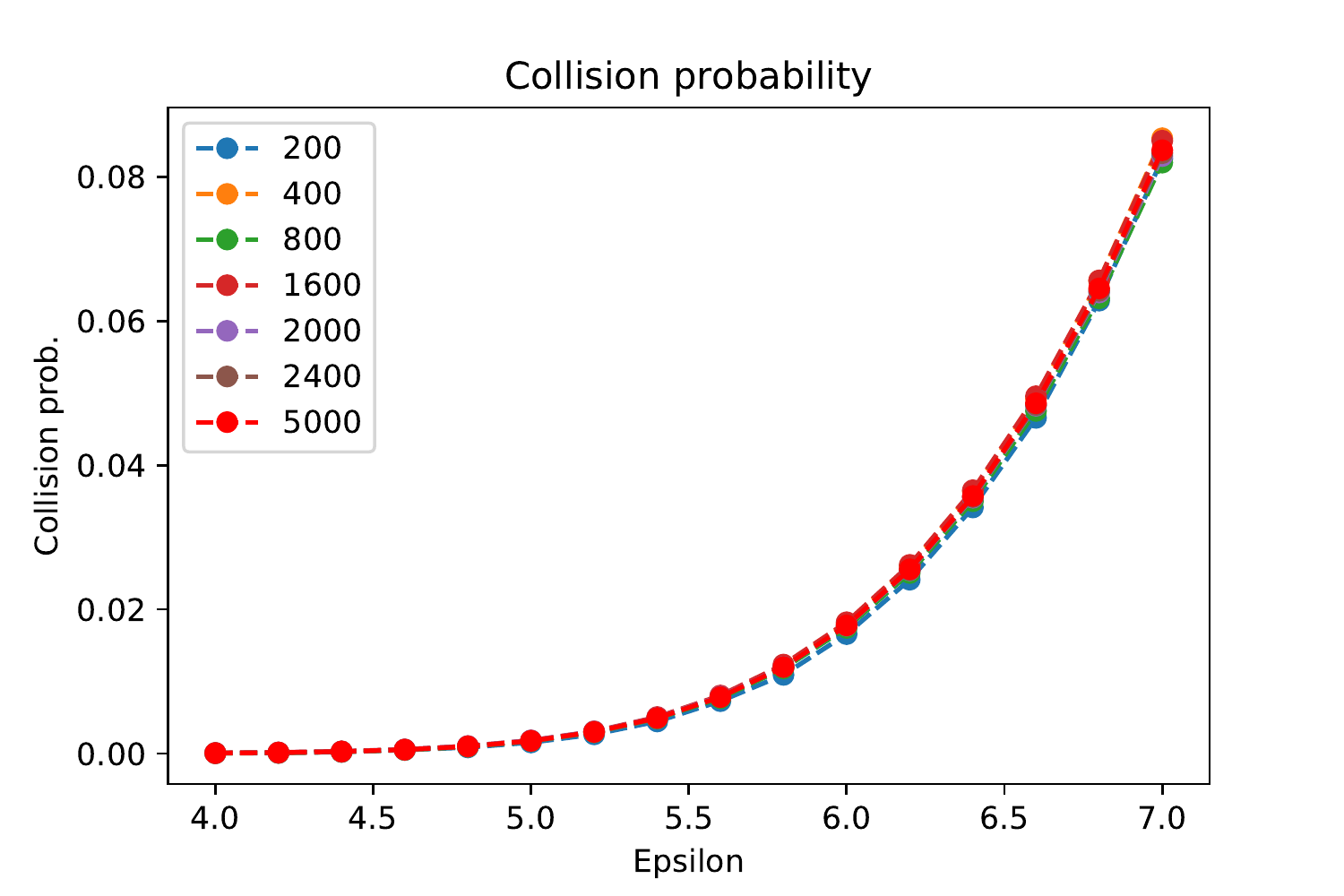}
    \caption{Variation in collision probability with attacker budget $\epsilon$ for the CIFAR-10 dataset}
    \label{fig:cifar-10_collide}
\end{figure}

\subsection{Runtime analysis for other datasets}
In Figures \ref{suppfig: timing_mnist} and \ref{suppfig: timing_fmnist}, we show the variation in runtime for the algorithms to compute the lower bound on cross-entropy loss for the MNIST and Fashion MNIST datasets. Our custom Algorithm (Algorithm 1 in the main body) clearly outperforms the generic convex solver from CVXOPT.

\begin{figure}[t]
	\centering
	\subfloat[Scaling with sample size at $\epsilon=3.8$]{\resizebox{0.3\textwidth}{!}{	\input{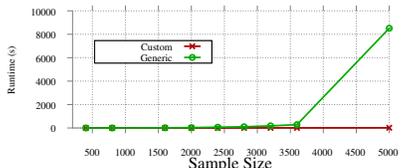}\label{subfig: time_vs_sample_mnist}}}
	\hspace{1pt}
	\subfloat[Scaling with $\epsilon$ for $5000$ samples per class]{\resizebox{0.3\textwidth}{!}{\input{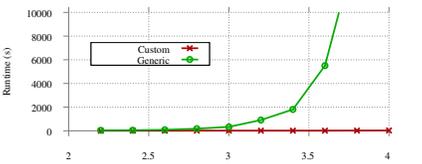}}\label{subfig: time_vs_eps_mnist}}
	\caption{Algorithm runtime comparisons for MNIST}
	\label{suppfig: timing_mnist}
\end{figure}

\begin{figure}[t]
	\centering
	\subfloat[Scaling with sample size at $\epsilon=4.0$]{\resizebox{0.3\textwidth}{!}{	\input{plots/fmnist_sample_var_}\label{subfig: time_vs_sample_fmnist}}}
	\hspace{1pt}
	\subfloat[Scaling with $\epsilon$ for $5000$ samples per class]{\resizebox{0.3\textwidth}{!}{\input{plots/fmnist_eps_var_}}\label{subfig: time_vs_eps_fmnist}}
	\caption{Algorithm runtime comparisons for Fashion MNIST}
	\label{suppfig: timing_fmnist}
\end{figure}

\begin{figure}[t]
	\centering
	\subfloat[MNIST]{\resizebox{0.48\textwidth}{!}{\input{plots/3_7_mnist_l2_5000_01_loss_compare_resnet_}}\label{subfig: mnist_KL_01}}
	\hspace{0mm}
	\subfloat[Fashion MNIST]{\resizebox{0.48\textwidth}{!}{\input{plots/3_7_fmnist_l2_5000_01_loss_compare_resnet_}}\label{subfig: fmnist_KL_01}}
	\caption{Comparison on \emph{training data} between the $0-1$ loss obtained by different training methods (computed using AutoAttack) versus the optimal loss.}
	\label{fig: compare_plot_01}
	\vspace{-10pt}
\end{figure}

\begin{figure}[t]
	\centering
	\subfloat[MNIST]{\resizebox{0.48\textwidth}{!}{\input{plots/3_7_mnist_l2_5000_01_loss_compare_resnet_test_}}\label{subfig: mnist_KL_test}}
	\hspace{0mm}
	\subfloat[Fashion MNIST]{\resizebox{0.48\textwidth}{!}{\input{plots/3_7_fmnist_l2_5000_01_loss_compare_resnet_test_}}\label{subfig: fmnist_KL_test}}
	\caption{Comparison on \emph{test data} between the $0-1$ loss obtained by different training methods (computed using AutoAttack) versus the optimal loss.}
	\label{fig: compare_plot_01_test}
	\vspace{-10pt}
\end{figure}

\subsection{Further Gaussian results}
In Table \ref{fig:syn_gauss_plot_d2}, we show the variation in the population- and sample-level lower bounds on the cross-entropy loss for data generated from a 2-class Gaussian mixture with $d=2$. All other parameters are the same as in Section 4.2 of the main paper. We can see that for lower dimensional data, the gap between the bounds is small.

\begin{figure}[t]
    \centering
    \input{plots/syn_gauss_d2_}
    \caption{Comparing the population-level and sample-level lower bounds on cross-entropy loss for synthetic $2$-class Gaussian data of dimension $2$.}
    \label{fig:syn_gauss_plot_d2}
    \vspace{-10pt}
\end{figure}

\subsection{Minimum $0-1$ loss}
 We note that the optimal classifier probabilities that are obtained in the course of determining the minimum log-loss can be thresholded to obtain the classification outcomes of the optimal classifier. Care must be taken, however, for data points where the optimal probability is $\frac{1}{2}$ in the two class case. For all points of this type, we just classify them as being in class $1$, which avoids any conflicts and recovers the numerical values from previous work \cite{bhagoji2019lower}. These bounds are plotted in Figure \ref{fig: compare_plot_01} as the line `Minimum loss'.

\section{Further Results on Robust Training} \label{appsec: robust_train_extra}

\subsection{Robust $0-1$ loss}
 We also compare the minimum possible $0-1$ loss to that obtained by various robust training methods using AutoAttack \cite{croce2020reliable} for both training (Figure \ref{fig: compare_plot_01}) and test (Figure \ref{fig: compare_plot_01_test}) data. We find that robust training using optimal clipped soft labels can outperform standard hard label training, and that TRADES performs poorly at higher adversarial budgets.



\subsection{Ablation}
\noindent \textbf{Activation functions:} In Table \ref{tab: activation}, we study the variation in training and test cross-entropy loss with the activation functions used in a ResNet-18. We find at a budget of $3.0$ for MNIST, the standard ReLU activation function performs the best, justifying our choice of this activation function throughout. For the ELU and Tanh activation functions, the network is unable to converge, implying that not all activation functions perform well at higher budgets.

\begin{table}[t]
\centering
\begin{tabular}{@{}lll@{}}
\toprule
Activation function & Robust train loss & Robust test loss \\ \midrule
ReLU       & 0.106      & 0.236     \\
ELU        & 1.056      & 1.060     \\
Tanh       & 13.012     & 13.099    \\
Leaky ReLU & 0.103      & 0.348     \\
SELU       & 0.704      & 0.706 \\  \bottomrule
\end{tabular}
\caption{Variation in train and test loss for a ResNet-18 trained on MNIST with an $\ell_2$ norm adversary with $\epsilon=3.0$}
\label{tab: activation}
\end{table}

\noindent \textbf{Architecture:} We also experimented with different ResNet architectures to test if increasing the size of the network would lead to lower values of the robust cross-entropy loss. However, in Table \ref{tab: arch}, we find that while the loss varies across architectures, an increase in size is not guaranteed to even lower the training loss.

\begin{table}[t]
\centering
\begin{tabular}{@{}lll@{}}
\toprule
Architecture & Robust train loss & Robust test loss \\ \midrule
ResNet-18    & 0.451             & 0.451            \\
ResNet-50    & 0.387             & 0.387            \\
ResNet-101   & 0.422             & 0.425          \\ \bottomrule
\end{tabular}%
\caption{Variation in train and test loss for models trained on Fashion MNIST with an $\ell_2$ norm adversary with $\epsilon=5.0$}
\label{tab: arch}
\end{table}

\subsection{CIFAR-10 robust training}
We robustly train a ResNet-18 on the CIFAR-10 dataset using $\ell_2$ budgets of $\epsilon=1.0$ and $2.0$. We find that at $\epsilon=1.0$, the training loss with both adversarial training and TRADES goes to $0$, but the test loss is around $1$, with a robust classification accuracy of just above $50\%$, implying that some robust learning is just about possible.

When the budget increases to $\epsilon=2.0$, the network has below $50\%$ robust classification accuracy on the test set for both training methods. Thus, the performance of current robust classifiers is very far from the optimal cross-entropy lower bound of $0.0$ at both these budgets.



\end{document}